\newif\ifICML\ICMLfalse

\ifICML
\documentclass{article}
\else
\documentclass{llncs}
\fi

\usepackage{microtype}
\usepackage{graphicx}
\usepackage{subfigure}
\usepackage{booktabs} 

\usepackage{breakurl}
\usepackage[breaklinks]{hyperref}

\newcommand{\infinity}{\infty}
\newcommand{\remove}[1]{}
\newcommand{\point}[1]{\mathbf{#1}}
\newcommand{\pointx}{\point{x}}
\newcommand{\pointy}{\point{y}}
\newcommand{\pointz}{\point{z}}

\ifICML
\usepackage{amsmath,amsthm,amssymb}

\newtheorem{theorem}{Theorem}
\newtheorem{lemma}{Lemma}

\usepackage[accepted]{icml2019}

\icmltitlerunning{Submission and Formatting Instructions for ICML 2019}

\begin{document}

\twocolumn[
\icmltitle{A Simple Explanation for the Existence of Adversarial Examples\\
with Small Hamming Distance}



\icmlsetsymbol{equal}{*}

\begin{icmlauthorlist}
\icmlauthor{Adi Shamir}{wi}
\icmlauthor{Itay Safran}{wi}
\icmlauthor{Eyal Ronen}{tau}
\icmlauthor{Orr Dunkelman}{hu}

\end{icmlauthorlist}

\icmlaffiliation{wi}{Weizmann Institute of Science, Rehovot, Israel}
\icmlaffiliation{tau}{Tel Aviv University, Tal Aviv, Israel}
\icmlaffiliation{hu}{Haifa University, Haifa, Israel}

\icmlcorrespondingauthor{Adi Shamir}{adi.shamir@weizmann.ac.il}
\icmlcorrespondingauthor{Itay Safran}{itay.safran@weizmann.ac.il}

\icmlkeywords{Machine Learning, ICML}

\vskip 0.3in
]



\printAffiliationsAndNotice{}  
\else
\usepackage{amsmath,amsfonts}
\usepackage{natbib}
\usepackage{algorithm}
\usepackage{algorithmic}
\begin{document}
\title{A Simple Explanation for the Existence of Adversarial Examples\\
with Small Hamming Distance}

\author{Adi Shamir\inst{1}
	\and Itay Safran\inst{1}
	\and Eyal Ronen\inst{2}
	\and Orr Dunkelman\inst{3}}
\institute{%
	Computer Science Department, The Weizmann Institute, Rehovot, Israel
	\and
	Computer Science Department, Tel Aviv University, Tel Aviv, Israel
	\and
	Computer Science Department, University of Haifa, Israel\\
}

\maketitle
\fi
\begin{abstract}
The existence of adversarial examples in which an imperceptible change in the input can fool well trained neural networks
was experimentally discovered by Szegedy et al in 2013, who called them ``Intriguing properties of neural networks''.
Since then, this topic had become one of the hottest research areas within machine learning, but the ease with which we can switch
between any two decisions in targeted attacks is still far from being understood, and in particular it is not clear
which parameters determine the number of input coordinates we have to change in order to mislead the network.
In this paper we develop a simple mathematical framework which enables us to think about this baffling phenomenon from
a fresh perspective, turning it into a natural consequence of the geometry of $\mathbb{R}^n$ with the $L_0$ (Hamming) metric,
which can be quantitatively analyzed. In particular, we explain why we should expect to find targeted adversarial examples
with Hamming distance of roughly $m$ in arbitrarily deep neural networks which are designed to distinguish between $m$ input classes.
\end{abstract}

\section{Introduction}
\label{sec:Introduction}

Adversarial examples in neural networks were first described in the seminal paper of~\citet{szegedy2013intriguing}. It was followed by a large and rapidly increasing literature on the applications and implications of such examples, which concentrated primarily on the issue of how to generate them more efficiently, and on the dual issues of how to exploit them and how to avoid them (see Fig.~\ref{fig:tabby_gauacamole} for an example). Our goal in this paper is different -- we describe a simple mathematical framework which explains why adversarial examples with a small Hamming distance are a natural byproduct (rather than a freak phenomenon) whenever we partition the high dimensional input space $\mathbb{R}^n$ into a bounded number of labelled regions by neural networks.

\begin{figure}[ht]
	\vskip 0.2in
	\begin{center}
		\centerline{\includegraphics[width=\columnwidth]{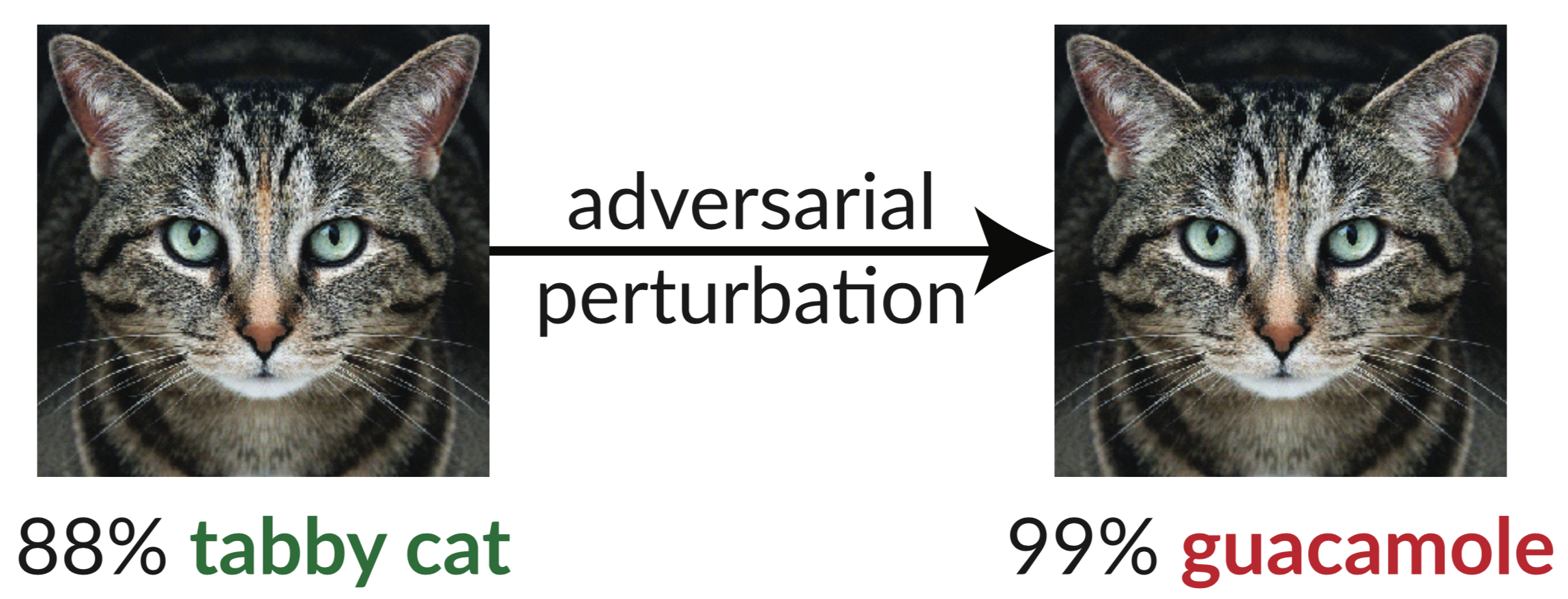}}
		\vskip -0.1in
		\caption{A small change imperceptible to humans misleads the InceptionV3 network into classifying an image of a tabby cat as guacamole. Image taken from https://github.com/anishathalye/obfuscated-gradients.}
		\label{fig:tabby_gauacamole}
	\end{center}
	\vskip -0.2in
\end{figure}

Previous attempts to explain the existence of adversarial examples dealt primarily with untargeted attacks, in which we are given a neural network and some input $\pointx$ which is classified by the network as belonging to class $C_1$, and want to find some $\pointy$ with $distance(\pointx,\pointy)< \epsilon$ which is classified as belonging to a different class.
In this paper we consider the far stronger model of targeted attack, in which we are given any two classes $C_1$ and $C_2$, along with any point $\pointx \in C_1$, and our goal is to find some nearby $\pointy$ which is inside $C_2$. Our ability to do so suggests that all the classes defined by neural networks are intertwined in a fractal-like way so that any point in any class is simultaneously close to all the boundaries with all the other classes, which seems to be very counterintuitive and mysterious.

To clarify this mystery, we describe in this paper a new algorithm for producing adversarial examples, which has no practical advantages over existing gradient based techniques: it is more complicated, it is much slower, it requires full (white box) knowledge of the neural network, it uses infinitely precise real numbers, etc. However, it will be easy to quantitatively analyze its behaviour and to understand why it produces adversarial examples with a particular Hamming distance. To simplify our analysis, we eliminate all the domain-specific constraints imposed by real-world considerations. In particular, we assume that all the values we deal with (inputs, outputs, weights, etc) are real numbers which can range between $-\infinity$ and $ + \infinity $, and thus our analysis cannot be directly applied to particular domains such as images, in which all the grey levels must be integers which can only range in $[0,255]$.

Most of the adversarial examples found in the literature use one of two natural distance functions to measure the proximity between $\pointx$ and $\pointy$: either the $L_2$ norm which measures the Euclidean distance between the two vectors, or the $L_0$ ``norm''\footnote{We put the word norm in quotation marks since $L_0$ does not satisfy the requirement that $L_0(c\cdot\pointx)=cL_0(\pointx)$, but the distance function it defines satisfies all the other requirements such as the triangle inequality.} which measures the Hamming distance between the two vectors (ie, how many input coordinates were changed).  When applied to an input vector, the first type of adversarial examples are allowed to change all the input coordinates but each one of them only by a tiny amount, whereas the second type of adversarial examples are allowed to change only a few input coordinates, but each one of them can change a lot. In this paper we make essential use of the properties of $L_0$, and thus our arguments cannot be directly applied to the $L_2$ or any other norm.

Finally, for the sake of simplicity we consider in this paper only neural networks that use the ReLU function (which replaces negative values by zero and keeps positive values unchanged) as their source of nonlinearity (see Fig.~\ref{fig:dnn_max} for an illustration), but we can easily extend our results to other piecewise linear functions. In terms of architecture, we allow the neural networks to have arbitrary depth and width, to use any number of convolutional or maxpooling layers that are commonly used in practice, and to use an arbitrary label choosing mechanism including the popular softmax activation.

\begin{figure}[ht]
	\vskip -0.2in
	\begin{center}
		\centerline{\includegraphics[width=0.75\columnwidth]{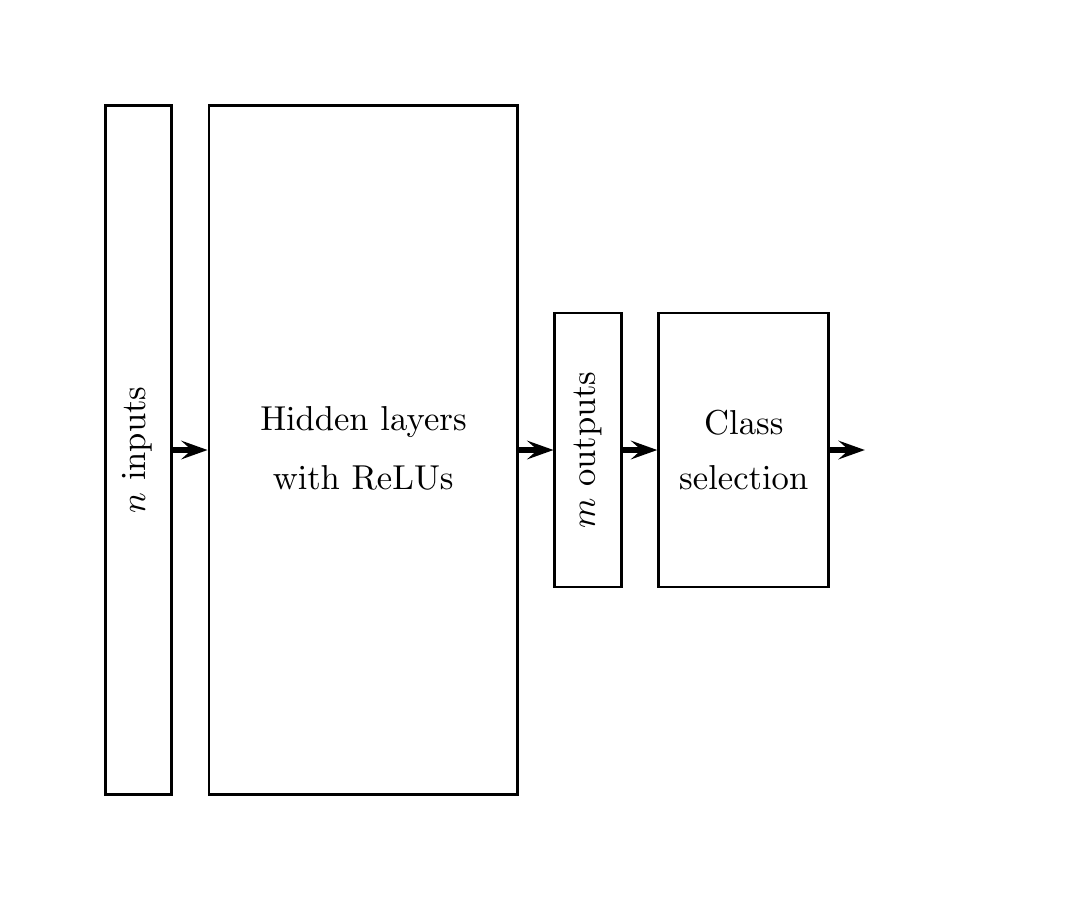}}
		\vskip -0.2in
		\caption{Our results apply to arbitrarily deep architectures employing a ReLU activation, including convolutional layers and maxpooling layers, and using an arbitrary class selection procedure.}
		\label{fig:dnn_max}
	\end{center}
	\vskip -0.2in
\end{figure}

To demonstrate the kinds of results we get in our quantitative analysis of the problem, consider an arbitrary partition of
$\mathbb{R}^n$ into a million regions with twenty linear separators. Our analysis shows the surprising result that changing just two coordinates (eg, moving left then up) typically suffices to move from any point in a given region to some point in another given region whenever the number of coordinates $n$ exceeds about $250$, and we verified this prediction experimentally for a large number of separators and cell pairs. When we consider the general case of an arbitrary
neural network which is designed to distinguish between $m$ possible classes, we provide a simple explanation why it should suffice
to change only about $m$ of the $n$ input coordinates in order to be able to move from any point in any region to some point in any other region. Once again, we confirmed our analysis by running actual experiments on the MNIST network (for which $n=784$ and $m=10$), which found adversarial examples in which different modifications in the same subset of $11$ pixels was sufficient in order to switch the network's decision from a high confidence in the original digit to a high confidence in any other digit.

\subsection{Related Work}

Surveying all the papers published so far about adversarial examples is well outside our scope, and thus we describe in this subsection only some of the more recent works, focusing for the most part on theoretical efforts to prove or explain the existence of adversarial examples.

Following the works of \citet{szegedy2013intriguing} and \citet{biggio2013evasion}, methods to produce adversarial examples for various machine learning tasks were developed, including computer vision, speech recognition and malware detection \citep{goodfellow6572explaining,carlini2018audio,grosse2016adversarial}, as well as applications \citep{DeepCAPTCHA}. Consequentially, several papers proposed various defensive techniques, including gradient masking \citep{papernot2017practical}, defensive distillation \citep{papernot2016distillation} and adversarial training \citep{madry2017towards,kolter2017provable}. However, these were eventually rendered ineffective by more sophisticated attacks \citep{carlini2017towards,athalye2018obfuscated}. Attempts to instead detect whether the given input was adversarially generated were also made \citep{feinman2017detecting,grosse2017statistical,metzen2017detecting,li2017adversarial} but were quickly broken as well \citep{carlini2017adversarial}, and it seems like the 'attackers' are winning the arms-race against the 'defenders'. In addition, there are currently known techniques capable of generating 3D-printed adversarial examples that are consistently misclassified from different angles \citep{athalye2018synthesizing}, posing a tangible threat for safety critical applications such as autonomous driving, and emphasizing the importance of gaining a better understanding of this phenomenon.


Recently, a few papers addressed the theoretical question of why adversarial examples exist in machine learning. \citet{schmidt2018adversarially} provides evidence that adversarially robust classifiers require more data by examining a distribution of a mixture of two Gaussians, whereas \citet{bubeck2018adversarial} suggests that the bottleneck is the computational intractability of training a robust classifier rather than information theoretical constraints.

A number of papers show the existence of adversarial examples by using strong concentration of measure tools from probability theory and geometry, taking advantage of the fact that in a high dimensional space most points $\pointx$ in a ``blob'' are located next to its boundary. However, such arguments can only be used in untargeted attacks since there is likely to be only one other ``blob'' which is next to $\pointx$ and across the boundary. In addition, for their argument to work, an assumption that each class has a large measure must be made, whereas we make no such assumptions. Moreover, all the naturally occurring examples of a particular class $C$ may be a tiny fraction of all the inputs which are recognized by the network as belonging to class $C$. If they are all located at the center, far away from any boundary, then the argument fails to explain why natural examples can be modified to get adversarial examples. Typical examples of such results are \citet{shafahi2018adversarial} which shows that any measurable predictor on the unit hypercube $ [0,1]^n $ will have an adversarial example if perturbed strongly enough, and \citet{fawzi2018adversarial} which shows that any predictor on a data distribution supported on a hypersphere is vulnerable to small $ L_2 $ noise. While \citet{fawzi2018adversarial} assume a distribution that arguably does not capture real life data, \citet{shafahi2018adversarial} provides a more widely applicable result which does not exploit any inherent structure of a certain predictor, but only guarantees the existence of an untargeted adversarial example in the $ L_0 $ metric when $ \theta(\sqrt{n}) $ pixels are perturbed. On the other hand, there is experimental evidence that in the cifar10 dataset, perturbing a single pixel suffices to change the prediction of a deep neural network on the majority of the test set \citep{su2017one}. Unlike these works which study untargeted attacks, in this paper we derive stronger results in the targeted regime, modifying the \emph{entire} confidence output vector to values of our choice by slightly perturbing \emph{any} given input. In addition, our analysis requires fewer pixels than \citet{shafahi2018adversarial}, albeit at the cost of possibly straying from the pixel range of $ [0,255] $.

\section{Targeted Adversarial Examples in Linear Partitions of $\mathbb{R}^n$}
\label{sec:Linear}

We start our analysis by considering the simplest case in which the mapping from inputs to outputs is determined by a hyperplane arrangement. Note that this setting differs from linear separators for multiclass classification, where the maximum over several linear classifiers is used to determine the classification. Instead, we are given
$m$ hyperplanes of the form $\sum_{j=1}^{n}{a_{i}^{j}x_{j}}+b_{i}$ for $i=1,\ldots ,m$ which split $\mathbb{R}^n$ into {\it cells}, and assume that they are in general position (if not, we can always perturb them slightly to prevent potentially problematic coincidences). If we denote by $M$ the $m \times n$ matrix whose entries are the $a_i^j$ coefficients and by $B$ the column vector whose entries are the $b_i$ constants, then each cell in the partition is defined by
a particular vector $S$ of $m$ $\pm$ signs, and consists of all the points $\pointx$ in $\mathbb{R}^n$ for which $M\pointx+B$ is a column vector of $m$ numbers whose signs are as specified in $S$. The maximal possible number of cells is $\sum_{i=0}^{n}\binom{m}{i}$ \citep{zaslavsky1975facing}. The predictor associates labels $C_t$ (which are typically object classes such as "horse" or "car") to some of the cells (see Fig.~\ref{fig:linearclassifier}). The union of the cells labeled by $C_t$ can be arbitrarily shaped (in particular it need not be connected, convex, or hole-free). The predictor classifies any input vector $(x_1, \ldots , x_n)$ by using the label of the cell in which it is located, if such a label exists. Note that even a small number of hyperplanes (such as 20) suffices to split $\mathbb{R}^n$ into a huge number of cells (more than a million), and thus can be potentially used to recognize a large number of object classes (anywhere between $1$ and $2^{20}$).

\begin{figure}[ht]
\vskip 0.2in
\begin{center}
\centerline{\includegraphics[width=\columnwidth]{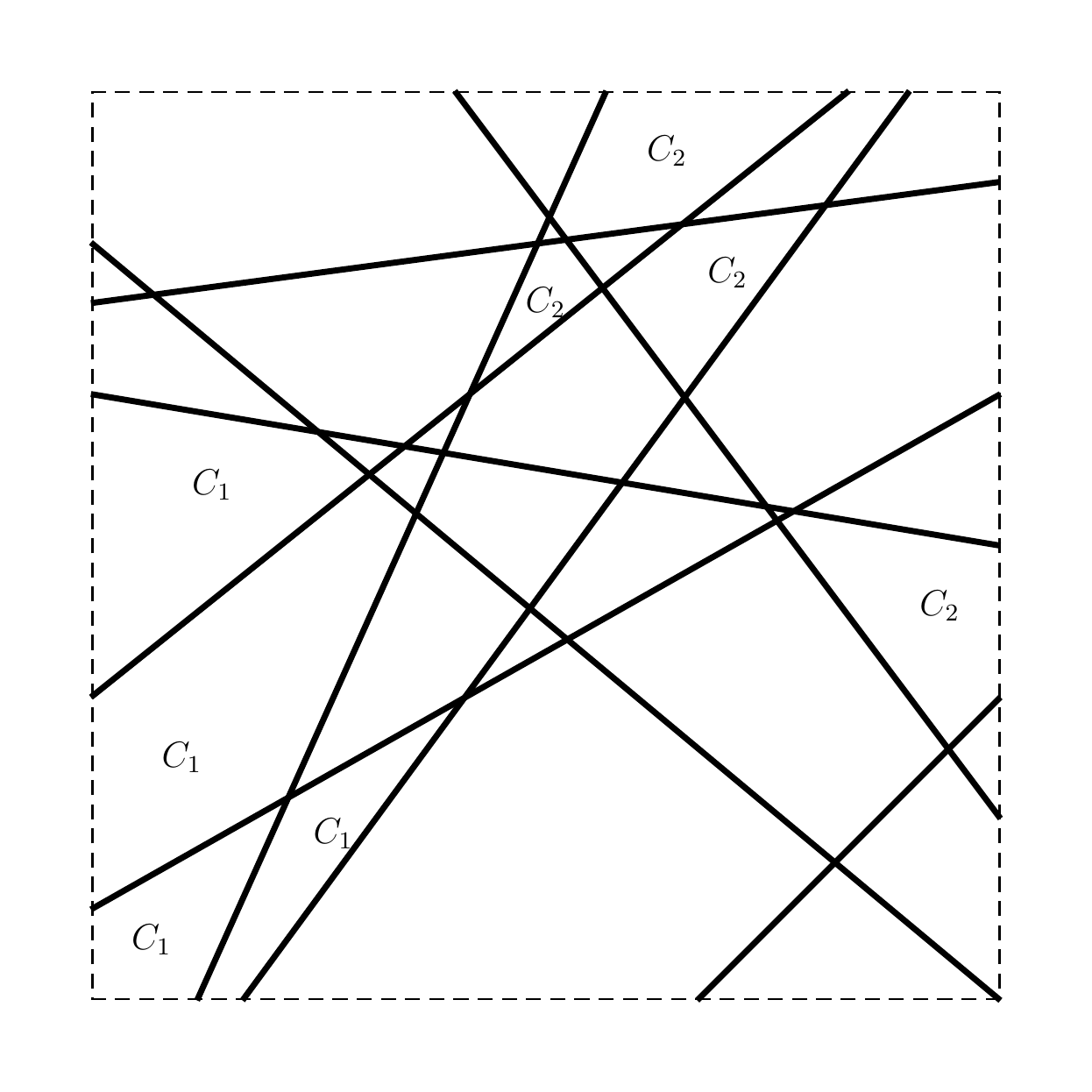}}
\vskip -0.2in
\caption{Separation of $\mathbb{R}^n$ by a hyperplane arrangement}
\label{fig:linearclassifier}
\end{center}
\vskip -0.2in
\end{figure}

Our goal is to study the neighborhood structure of such a partition of $\mathbb{R}^n$ under the $L_0$ (Hamming) distance function. More specifically, given any two cells and a particular point $\pointx$ which is located in the first cell, we would like to determine the smallest possible number $k$ of coordinates in $\pointx$ we have to change in order to reach some point $\pointy$ which is located in the second cell. It is easy to construct a worst case example in which the only way to move from the first cell to the second cell is to change all the coordinates (for example, when the first cell is defined by all the input coordinates being smaller than $1$, and the second cell is defined by all the input coordinates being larger than $2$). However, defining such cells requires a large number ($2n$) of hyperplanes. Our goal now is to show that the expected number $k$ of coordinates we have to change for moderately sized $m$ is surprisingly small.

\begin{theorem}
Consider any partition of $\mathbb{R}^n$ by $m$ hyperplanes defined by $M\point{x}+\point{b}$, and consider any pair of cells $C_1$ and $C_2$ which are defined by the sign vectors $S_1$ and $S_2$, respectively. Then a sufficient condition for being able to move from any $\pointx \in C_1$ to some $\pointy \in C_2$ by changing at most $k$ coordinates is the existence of a linear combination of at most $k$ columns in $M$ which is in the orthant of $\mathbb{R}^m$ specified by the sign vector $S_2$.
\end{theorem}

\begin{proof}
Assume without loss of generality that $S_1=(+,-,+,-, \ldots)$ and $S_2=(-,+,+,-,\ldots)$. We are allowed to change at most $k$ coordinates in $\pointx$, and would like to use this ability in order to simultaneously move from the positive side of the first hyperplane to its negative side, move from the negative side of the second hyperplane to its positive side, remain on the positive side of the third hyperplane, remain on the negative side of the fourth hyperplane, etc. It is thus necessary to decrease the value of the first linear expression and to increase the value of the second linear expression. As far as the next two expressions, a sufficient (but not necessary) condition is to increase the value of the third expression and to reduce the value of the fourth expression. More formally, assume that we can find a vector $\point{d}$ with at most $k$ nonzero entries such that the signs in $M\point{d}$ are as specified by $S_2$, and consider the vector $\pointy=\pointx+c\cdot\point{d}$ where $c$ is a free parameter. By linearity, $M\pointy+\point{b}=M\pointx+c\cdot M\point{d}+\point{b}$, and thus as we increase the parameter $c$ from $0$, the input $\pointy=\pointx+c\cdot\point{d}$ will move out of cell $C_1$, and eventually cross into cell $C_2$ for all sufficiently large value of $c$.
\end{proof}

We have thus reduced the problem of fooling a given predictor with a targeted attack into the following question: Given $n$, $m$ and $k$, can we expect to find such a sparse $\point{d}$? Our goal is to analyze the ``typical'' case, rather than pathological counterexamples which are not likely to happen in practice.

It is easy to analyse the particular values of $k=1,2$:

\begin{lemma}
Assume that the entries in the $m \times n$ matrix $M$ are randomly and independently chosen in a sign balanced way (ie, the probability of choosing a negative and positive values are the same). Then such a $\point{d}$ with $k=1$ is likely to exist whenever $n$ is larger than $2^m$.
\end{lemma}

\begin{proof}
Each column in $M$ is assumed to have a uniform distribution of its sign vector, and thus a particular sign vector $S$ will exist among the $n$ columns whenever $n$ exceeds the number of possible sign vectors which is $2^m$.
\end{proof}

For $m=20$, this requires the dimension $n$ of the input vector to be larger than about a million,which is large but not impractically large. We now show that for $k=2$ this bound on the input dimension drops dramatically. We first consider the following special case:

\begin{lemma}
Consider any two columns $\point{g}=(g_1, \ldots , g_m)$ and $\point{h}=(h_1, \ldots , h_m)$ in $M$ whose entries are nonzero. Then the two dimensional linear space spanned by them passes through exactly $2m$ orthants in $\mathbb{R}^m$.
\end{lemma}

\begin{proof}
Let $u,v$  be the coefficients of these vectors, and assume that we want the $m$ values $ug_i-vh_i$ for $1 \leq i \leq m$ to have particular signs. Dividing the $i$-th inequality by $vg_i$, we get $m$ inequalities of the form $u/v < h_i/g_i$ or $u/v > h_i/g_i$, depending on the desired sign as well as on the sign of $vg_i$. This is a set of $m$ conditions on the value of $u/v$, which is satisfiable if and only if all the specified lower bounds on $u/v$ are smaller than all the specified upper bounds. Without loss of generality, we can reorder the rows of the matrix so that the values of $g_i/h_i$ are monotonically increasing. We can then place $u/v$ in any interval between two consecutive values of $g_i/h_i$ and $g_{i+1}/h_{i+1}$, and for each choice work backwards what are the sign patterns of the two opposite orthants (one for positive $v$ and one for negative $v$)
which are crossed in this case. Assuming that no two fractions are the same, there are $m+1$ subintervals in which we can place $u/v$, and thus we get a total of $2m+2$ possible orthants, but two of them (corresponding to the extreme placements of $u/v$) had been counted twice. The total number of orthants which are crossed by the two dimensional subspace is thus $2m$.
\end{proof}

Since there are $n(n-1)/2$ possible choices of pairs of columns in $M$ and each one of them crosses $2m$ of the $2^m$ possible orthants, it seems plausible that $n$ should be larger than $\sqrt{2^m/m}$. However, we could not formally prove such a result since the orthants crossed by the two pairs of columns $(\point{g},\point{h}_1)$ and $(\point{g},\point{h}_2)$ are slightly correlated (since in both sets of orthants must contain the two orthants which are covered by $\point{g}$ alone), but extensive experimentation showed that it is an excellent approximation. In particular, for $m=20$ and $n=250$ (which is approximately $\sqrt{2^{20}/20}$), using all pairs of columns in a random matrix covered about $35 \%$ of the possible orthants, whereas in the fully uncorrelated case we could expect to cover about $70 \% $ of the orthants. In both cases this coverage increases rapidly towards $100 \% $ as n increases beyond $250$. Consequently, both the theory and our experiments indicate that for most practically occurring input dimensions, we can expect to be able to move from any point $\pointx$ in any one of the million cells to some point $\pointy$ in any other cell by changing just two coordinates in $\pointx$.

The problem of determining the number of orthants of $\mathbb{R}^m$ which are crossed by a $k$ dimensional linear subspace is closely related to the problem of determining the signrank of matrices (ie, determining the smallest possible rank of any matrix whose entries have particular signs), which was studied in~\citet{Alon}. A direct proof can be found in~\citet{stackexchange}, where it is shown that this number is ${2\sum_{d=0}^{k-1}} \binom{m-1}{d}$ which can be approximated as $m^{(k-1)}$. If we again make the (nonrigorous) assumption that the coverage provided by different choices of $k$ out of the $n$ columns of $M$ are independent, we get the general asymptotic bound that a Hamming distance of $k$ should be sufficient in targeted attacks on hyperplane arrangement based models whenever the input dimension $n$ exceeds $2^{m/k}/m^{(k-1)/k}$.

\section{Adversarial Examples in Deep Neural Networks}
\label{sec:DNN}

\subsection{Our Basic Algorithm}

In this subsection we describe a new conceptual way to generate adversarial examples which can be quantitatively analyzed, and in particular it will be easy to understand why it creates adversarial examples with a particular Hamming distance.

We assume that we are given a fully trained deep neural network with $t$ ReLU's which assigns to each input $\pointx \in \mathbb{R}^n$  a vector of $m$ real numbers, which describe the network's confidence levels that x belongs to each one of the $m$ classes. We can thus view the neural network as a piecewise linear mapping from $\mathbb{R}^n$ to $\mathbb{R}^m$, and assume that $m\ll n$. Our algorithm is independent of the actual labeling procedure of $\pointx$ used by the neural network (e.g., softmax, rejection of examples with small confidence level). \remove{ or any arbitrarily function of the $m$ values).}

The basic idea in our proof is to consider any two input vectors $\pointx$ and $\pointy$ (which can be completely different in all their coordinates) whose labels are $C_1$ and $C_2$, and to gradually change $\pointx$ into a new $\pointz$ which will get exactly the same vector of confidence levels as $\pointy$, and thus will also be labeled by the given neural network as being in class $C_2$. We start our algorithm with $\pointx_1=\pointx$, and in the $i$-th step of the algorithm we slightly modify the current $\pointx_i$ into $\pointx_{i+1}$ by changing only about $m$ of its $n$ coordinates, until we reach $\pointx_p=\pointz$. The crucial property of the $L_0$ norm we use is that if we always change the same input coordinates, then regardless of how many times we change these coordinates and by how much we change them, the Hamming distance between the endpoints $\pointx$ and $\pointz$ of the chain will remain bounded by about $m$, and thus $\pointz$ will be an adversarial example which is very close to $\pointx$ and yet fools the network into switching the label from $C_1$ to $C_2$. Note that if we replace $L_0$ by $L_2$ (or any other norm), it is no longer true that the accumulated effect of a huge number of tiny changes remains tiny, and thus our argument will fail. In addition, we use the fact that the neural network is memoryless, and thus its final decision depends only on the final $\pointz$ we provide, and not on the intermediate $\pointx_i$ we pass through.

Our first step is to smoothly transition between $\pointx$ and $\pointy$ by linearly combining the two vectors. Define $$\pointx_{\alpha}=(1-\alpha) \cdot \pointx + \alpha \cdot \pointy$$ This connects $\pointx$ and $\pointy$ by a straight line in the input space as $\alpha$ ranges between zero and one. The mapping defined by the neural network is piecewise linear, and thus the straight line in the $n$ dimensional input space is mapped into a piecewise linear line in the $m$-dimensional output space, as depicted in Fig.~\ref{fig:Input}. This line consists of a finite number of straight line segments, where each change of direction is caused by one of the ReLU's switching between a negative and a positive input. Note that the output space can contain forbidden regions (such as the hatched region in the figure) which cannot be reached from any input vector, but the path we chose is guaranteed to avoid such regions since all the points along it are actually generated by some input vectors.

\begin{figure}[ht!]
\vskip -0.2in
\begin{center}
\centerline{\includegraphics[width=\columnwidth]{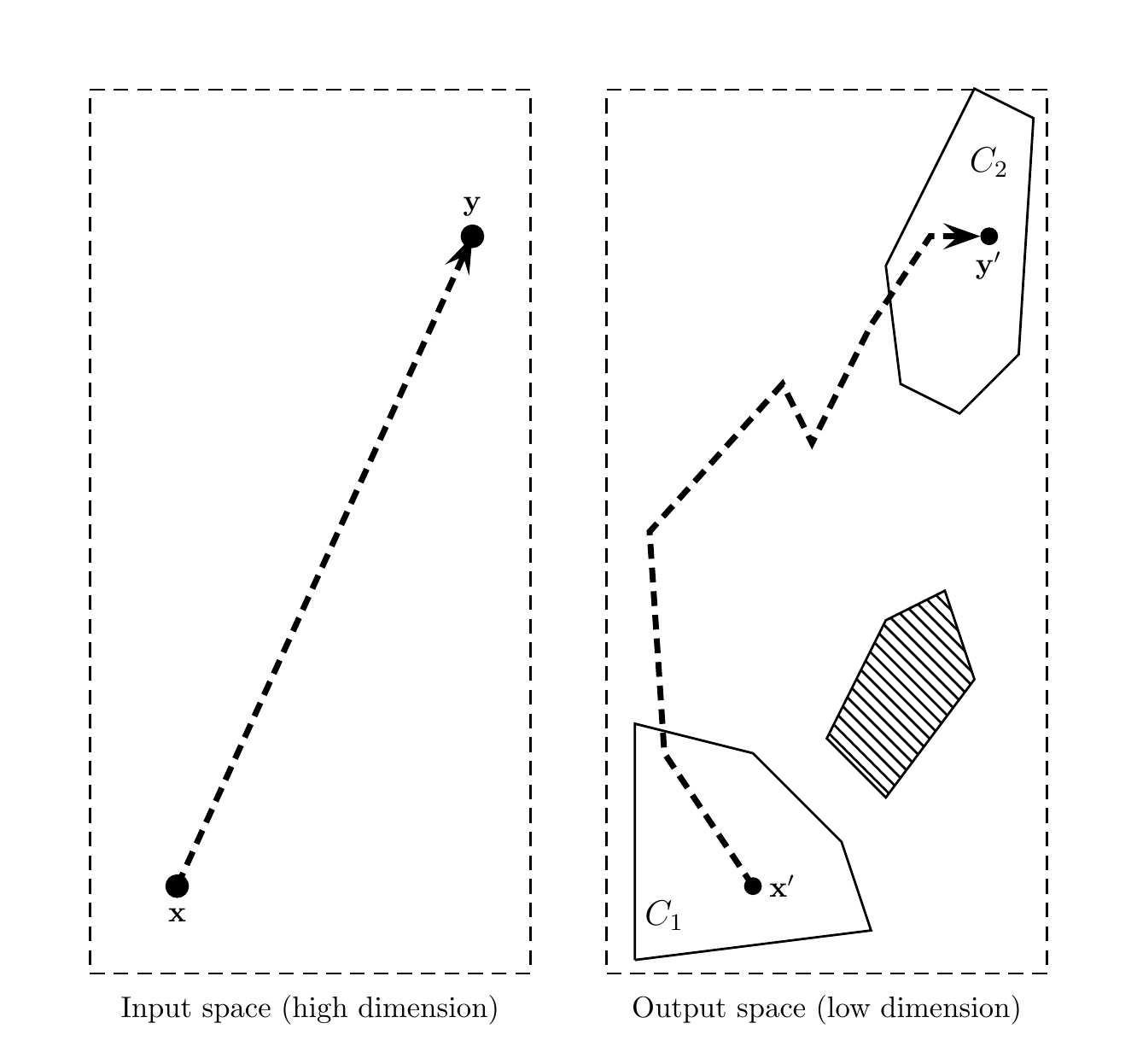}}
\caption{Movement in Input Space}
\label{fig:Input}
\end{center}
\vskip -0.2in
\end{figure}

\begin{figure}[ht!]
	\vskip -0.2in
	\begin{center}
		\centerline{\includegraphics[width=\columnwidth]{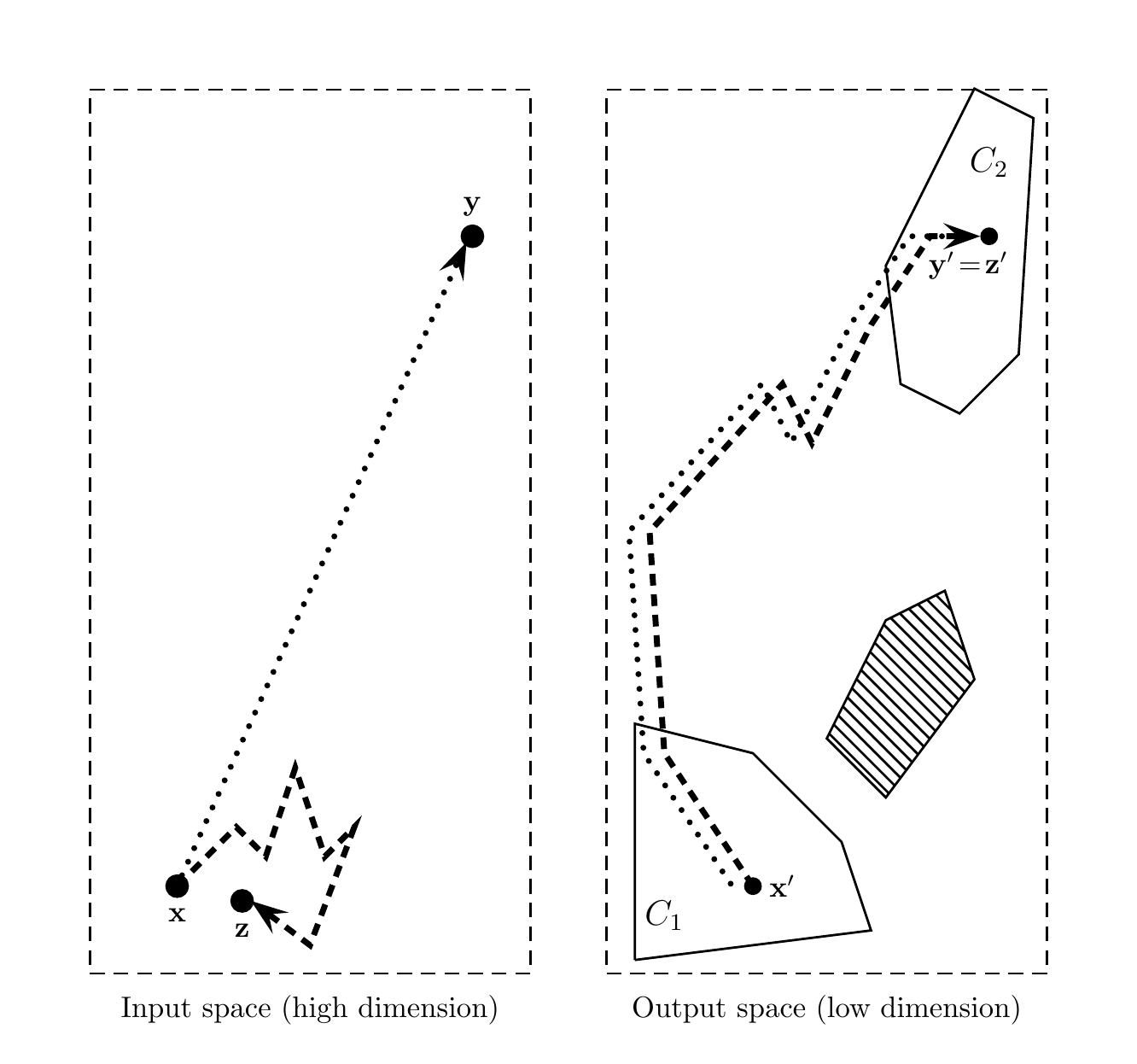}}
		\caption{Movement in Output Space}
		\label{fig:Output}
	\end{center}
	\vskip -0.2in
\end{figure}

\begin{algorithm}[ht]
	\caption{Basic Algorithm}
	\label{alg:basic}
	\begin{algorithmic}[1]
		\STATE {\bfseries Input:} $\pointx\in C_1$, target $C_2$, neural network $NN$.
		\STATE Compute $\pointx' = NN(\pointx)$.
		\STATE Pick an arbitrary $\pointy \in C_2$, and compute $\pointy'=NN(\pointy)$ in the output space.
		\STATE Connect $\pointx$ and $\pointy$ in a straight line in the input space as defined by $\pointx_{\alpha}$.
		\STATE Map the line (using $NN$) into a piecewise linear line between $\pointx'$ and $\pointy'$, denoted by $path$.
		\STATE Set $tmp \leftarrow \pointx$ and $tmp' \leftarrow \pointx'$.
		\STATE Choose an arbitrary subset of $m$ (out of $n$) input variables.
		\REPEAT
		\STATE Describe the linear map at the vicinity of $tmp$ to $tmp'$ as an $m\times m$ matrix $M'$.
		\STATE Find the direction in the reduced input space using $M'^{-1}$  that follows $path$.
		\STATE Advance $tmp$ and $tmp'$ in the direction of the path until a ReLU boundary is found, $path$ changes direction, or $\pointy'$ is reached.
		\UNTIL{$tmp' = \pointy'$}
	\end{algorithmic}
\end{algorithm}

Consider now any subset of $m$ coordinates in the input vector $\pointx$, and assume that only these coordinates in $\pointx$ are allowed to change, and all the other coordinates are frozen in their original values. This effectively reduces the input dimension from $n$ to $m$, and thus we get a piecewise linear mapping from an $m$-dimensional input space to an $m$-dimensional output space. Assuming that all the linear mappings we encounter are invertible, we can use the local inverse mappings to assign to the piecewise linear line in the output space a unique piecewise linear line in the reduced input space which will map to it under the piecewise linear map defined by the given neural network, as described in Fig.~\ref{fig:Output}. The resulting algorithm is given in Alg.~\ref{alg:basic}.

\subsection{Potential obstacles}

There are two potential modes of failure in this algorithm: Either we get stuck after a finite number of steps (a {\it hard failure}), or we run forever without reaching the destination in the output space (a {\it soft failure}).
\begin{figure}[ht!]
	\vskip -0.2in
	\begin{center}
		\subfigure[Reflection at a ReLU Boundary\label{fig:reflection}]{\includegraphics[width=0.45\columnwidth]{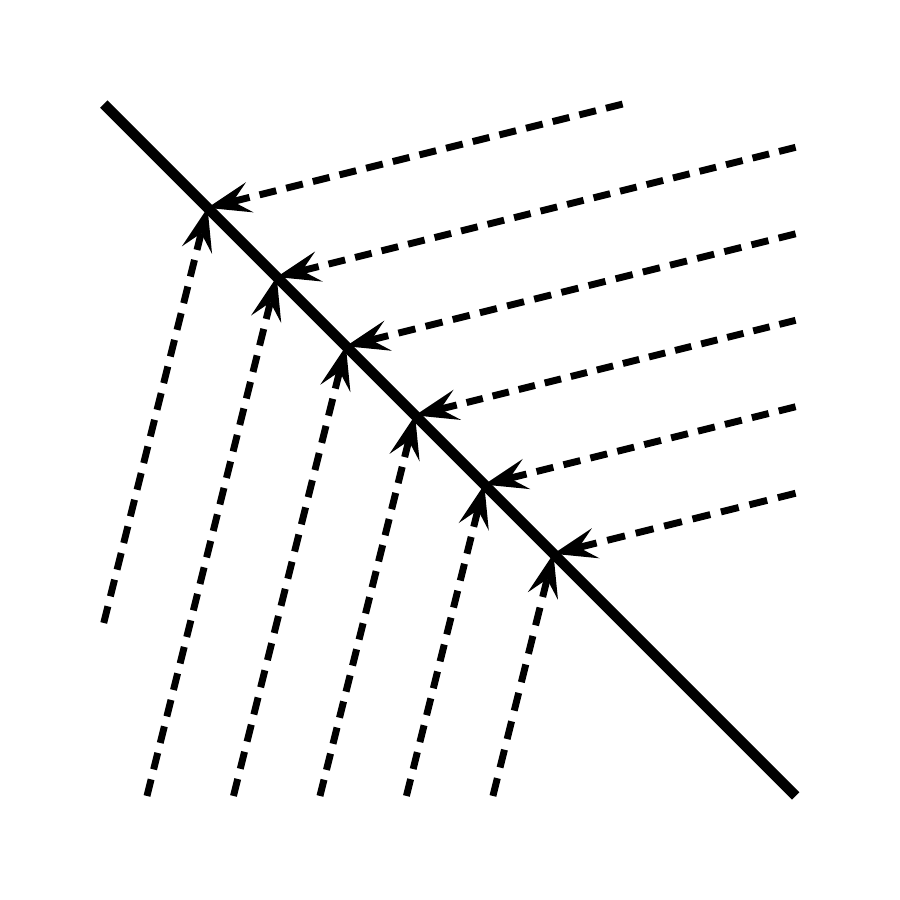}}~
		\subfigure[A Sinkhole\label{fig:sinkhole}]{\includegraphics[width=0.45\columnwidth]{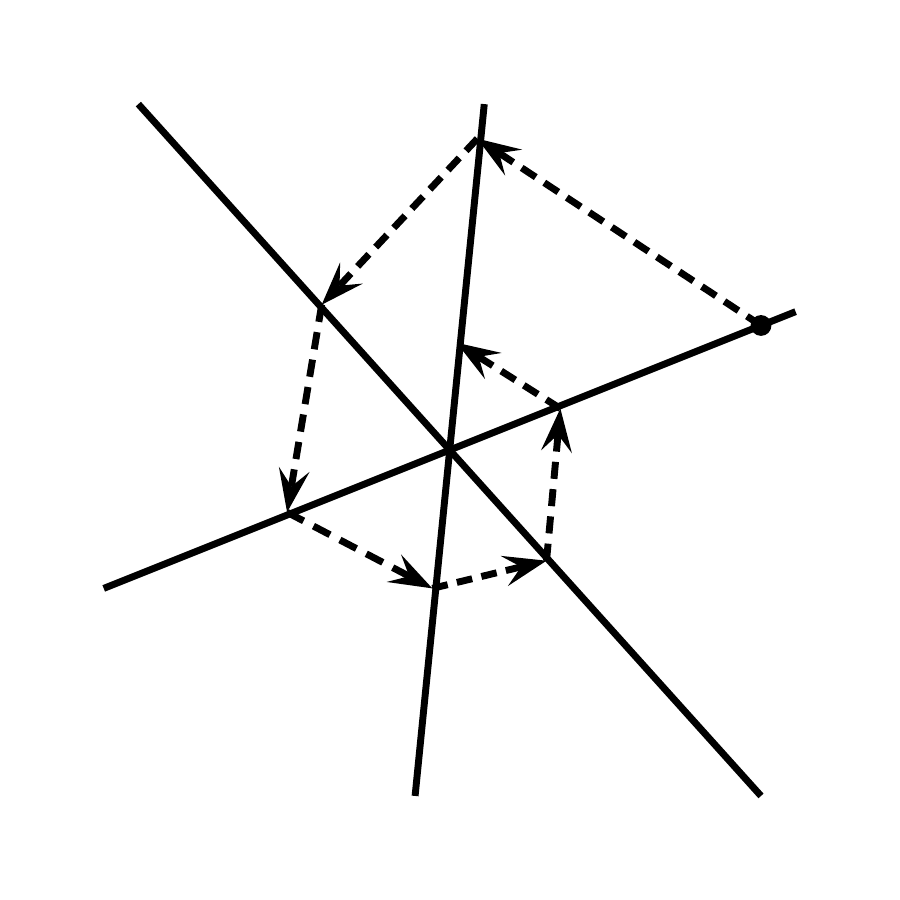}}
		\caption{Possible Obstacles}\label{fig:obstacle}
		
	\end{center}
	\vskip -0.2in
\end{figure}

\subsubsection{Hard Failures}
A hard failure can happen in two possible ways: Either there is no direction we can follow in the input space, or there is no way to cross some ReLU boundary. The first situation happens when we encounter a singular matrix $M'$ so that its inverse cannot be applied to the direction we want to follow in the output space. However, this is extremely unlikely to happen, since the probability that a random $m \times m$ matrix of real numbers with infinite precision will be singular is zero, and even when it happens any small perturbation of its entries will make it invertible. In addition, there are only finitely many possible linear input/output mappings which can be represented by a given network (since its $t$ ReLU's can be in at most $2^t$ combinations of sides), and finitely many ways to choose a subset of $m$ out of the $n$ columns, and thus by the union bound the probability that any one of them will be singular is zero.

The second situation (which we call {\it reflection}) happens when we are on one side of a ReLU boundary, move towards it, but as soon as we cross the boundary to the other side the linear mapping changes so that we now want to move back towards the original side. This situation (which is depicted in Fig.~\ref{fig:reflection}) gets us stuck at the boundary, without being able to proceed to either side. \remove{This situation is not expected to happen often since in a high dimensional space we usually move almost perpendicularly towards any boundary, and switching a single ReLU in the middle of the network tends to deflect
the direction of movement only by a small amount after we cross the boundary, but we did encounter it in our experiments.}

\subsubsection{Overcoming Reflection-Type Hard Failures}
An easy way to avoid any reflection-type hard failure is to allow $m+\Delta$ input variables to change instead of $m$, where $\Delta$ is a small positive integer. In this case the relevant part of the input space is of dimension $m+\Delta$, and the dimension of the output space remains $m$. Assume that we have just crossed a ReLU boundary at the $m+\Delta$ dimensional point $\pointx_i$ in the input space, and want to move towards the $m$-dimensional point $\pointx'_i$ in the output space. Assuming that the current matrix has full rank, we can find the subspace in the input space that is mapped to $\pointx'_i$ by the linear mapping which is in effect on the far side of the ReLU \remove{(e.g., for $\Delta=1$ this is a one dimensional line)}. Since we assume that all the boundaries are always in general position, this subspace will intersect the ReLU boundary hyperplane, and part of it will be on the far side of this boundary. We can thus pick any point $\pointx_{i+1}$ on this part of the subspace, and by going from $\pointx_i$ towards $\pointx_{i+1}$ in the input space, we will be able to avoid the reflection-type hard failure by moving towards the point $\pointx'_i$ in the output space using the correct linear mapping which is in effect in the region we traverse. In particular, when we choose $\Delta=1$, the subspace of points we want to move towards is a one dimensional line, and we can choose a random point on the ray that is cut from this line by the far side of the ReLU (see Fig.~\ref{fig:improvalg}). The changes to the original algorithm are described in Alg.~\ref{alg:improved}.

\begin{algorithm}[ht!]
	\caption{Improved Algorithm (replacing lines 7--12 of Algorithm~\ref{alg:basic})}
	\label{alg:improved}
	\begin{algorithmic}
		\STATE {\bfseries Input:} $\pointx\in C_1$, target $\pointy\in C_2$, neural network $NN$.
		\STATE $\pointx' = NN(\pointx)$, $\pointy'=NN(\pointy)$, $path$ from $\pointx'$ to $\pointy'$.
	\end{algorithmic}
	\begin{algorithmic}[1]
		\setcounter{ALC@line}{6}
		\STATE Choose an arbitrary subset of $m+\Delta$ (out of $n$) input variables.
		\REPEAT
		\STATE Describe the linear map at the vicinity of $tmp$ to $tmp'$ as an $(m+\Delta)\times m$ matrix $M'$.
		\STATE Find the $\Delta$-dimensional space in the reduced input space that follows $path$.
		\STATE Choose a direction to a random point in the part of this $\Delta$-dimensional space which was on the far side of the ReLU.
		\STATE Advance $tmp$ and $tmp'$ in the direction of the path until a ReLU boundary is found, $path$ changes direction, or $\pointy'$ is reached.
		\UNTIL{$tmp' = \pointy'$}
	\end{algorithmic}
\end{algorithm}

\begin{figure}[ht!]
	\vskip -0.2in
	\begin{center}
		\centerline{\includegraphics[width=1.15\columnwidth]{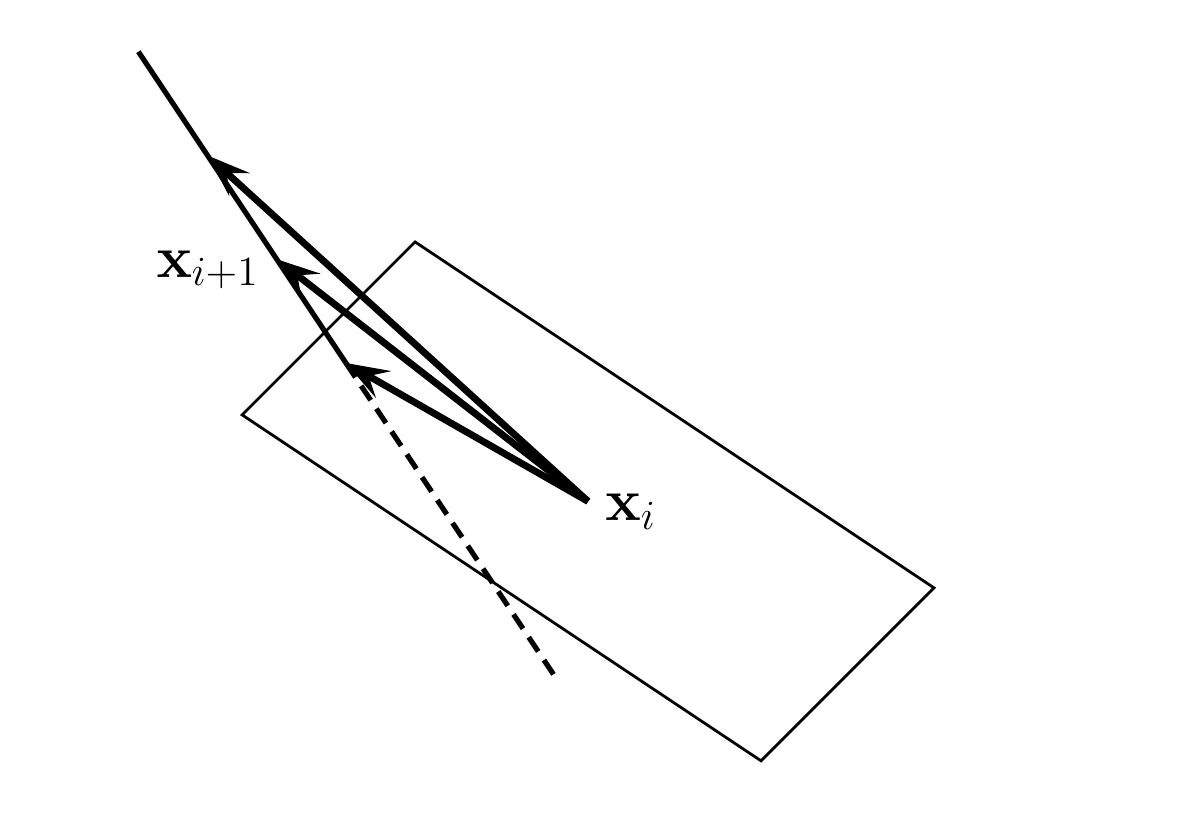}}
		\caption{Choosing a direction using $m+1$ input variables}
		\label{fig:improvalg}
	\end{center}
	\vskip -0.2in
\end{figure}

\subsubsection{Soft Failures}
Consider now how a soft failure can occur in Alg.~\ref{alg:basic}. One possibility is that the path in the input space will wander around without converging to a single limit point (e.g., by following a three dimensional helix with constant radius and diminishing pitch which converges towards a two dimensional limit circle). We claim that such a failure cannot happen. The reason is that there is a finite number of $m \times m$ invertible linear mappings we can encounter, and thus all their eigenvalues are bounded from above, from below, and away from zero. Consequently, the ratio between the lengths of any pair of corresponding paths in the input and output spaces is bounded by some constant, so we cannot have an infinitely long path in the input space which maps to  the finitely long partial path we manage to follow in the output space.

This leaves only one possible failure mode in the algorithm - performing an infinite number of moves of diminishing length in the input space (so that its total length is finite) which converges to some singular point along the path we try to follow in the output space. Such a situation (which we call a {\it sinkhole}) is depicted in Figure~\ref{fig:sinkhole}, where in each region the linear mapping tries to move perpendicularly to the ReLU boundary which is nearest it in counterclockwise direction. We can detect such a failure and artificially stop the algorithm by requiring that each step should move us at least $\epsilon$ closer to the destination in the output space for some tiny constant $\epsilon$. One way to remedy such a failure is to use the extra $\Delta$ degrees of freedom and our ability to randomize the algorithm in order to to explore new directions we can follow in the input space which can possibly bypass the problematic points along the output path. As demonstrated in Section \ref{sec:exp}, it was sufficient in our experiments to choose $\Delta=1$ in order to overcome both the hard and the soft failures.

\remove{In our experiments we did encounter situations in which restricting the Hamming distance to $m+1$ led to a sinkhole, but allowing distances between $m+2$ and $2m$ solved the problem and led to the creation of an adversarial example after finitely many steps.}
\remove{
The third obstacle (which we call a hard stop) is the existence of a point along the path in the output space which can be reached after a finite number of steps, but does not allow us to proceed any further. To demonstrate an extreme example of this phenomenon, consider the following choice of a function f, which maps $n$ input values to one output value:

$f(x_1,x_2, \ldots ,x_n) = \sum_{j=1}^n{x_j} - n \sum_{j=1}^n{|x_j-x_{j+1}|}$

where $x_{n+1}$ is defined as $x_1$. Since each $|x_j-x_{j+1}|$ value can be expressed as $ReLU(x_j-x_{j+1}) + ReLU(x_{j+1}-x_j)$, this function can be realized by an extremely simple neural network with a single hidden layer containing ReLU's followed by one linear output layer. Consider now the starting value of $X=(1,0,0, \ldots ,0)$, for which the value of $f$ is $1-2n$. Our goal is to increase the value of $f$ as much as possible, under the sole constraint that the first coordinate must remain unchanged at $1$. When we increase the value of $x_2$ in $\pointx$ from $0$ to $1$, we benefit moderately since the first term increases by $1$, while the second term remains unchanged (the changes in $|x_1-x_2|$ and $|x_2-x_3|$ cancel each other out). However, when we try to increase $x_2$ beyond $1$ to $1+\epsilon$, the moderate benefit turns into a sharp loss, since the first term increases by $\epsilon$, while the second term decreases by $n \epsilon$, so we must stop at $x_2=1$. Once $x_2$ had been increased to $1$, we can repeat the same argument about $x_3, x_4 \ldots, x_{n-1}$, etc., slowly increasing the value of $f$ from $1-2n$ to $(n-1)-2n$. Finally, when we also change $x_n$ from $0$ to $1$, the value of $f$ jumps from $-n-1$ to $n$, but no further increase is possible for any choice of $x_2,x_3, \ldots , x_n$ as long as $x_1=1$. This represents a hard stop in our algorithm, which exists even when we are allowed to change $n-1$ of the $n$ input coordinates. However, when we unfreeze $x_1$ and allow all the $n$ coordinates to change, we can reach any positive value $nc$ by setting all the input coordinates to $c$.

The reason we get stuck in this example is that to increase the value of $f$ we want to move towards some ReLU boundary, but as soon as we cross it to the other side the linear mapping changes sharply, so we want to move back towards the original side (as depicted in Fig.~\ref{fig:reflection}). This forces us to remain on the $n-1$ dimensional hyperplane of the ReLU, which reduces our degrees of freedom by $1$. Once we lose all our available degrees of freedom by having to remain on additional intersecting hyperplanes, we get completely stuck. Intuitively, this should be a rare case if the direction change encountered in the overall input/output mapping when a single ReLU switches is small.  However, we did encounter it occasionally in our experiments, but we always managed to overcome this obstacle by allowing a slightly larger Hamming distance (beyond $m$) in the targeted adversarial attack. This provided additional degrees of freedom to the choice of path in the input space, which made it possible to continue the attack. Providing a rigorous analysis of this potential obstacle and its solution seems to be a very difficult task, which we leave as a challenging open problem.
}

\section{Experimental results}\label{sec:exp}

In this section, we empirically explore the $ L_0 $ perturbation required for generating adversarial attacks on the MNIST dataset using Alg.~\ref{alg:basic} and Alg.~\ref{alg:improved}. More concretely, we trained a simple one-hidden layer ReLU network of width 256, for 200 epochs to 95.72\% accuracy on the test set using a Pytorch based git from https://github.com/junyuseu/pytorch-cifar-models. Thereafter, we ran 1000 instantiations of Alg.\footnote{Instead of projecting onto the ray (as in Fig~\ref{fig:improvalg}), we implemented a random projection onto the line, as this was sufficient to get satisfactory results.}~\ref{alg:improved}, attempting to follow the path using finite precision calculations, monitoring the distance of each instantiation from the output path as depicted in Fig.~\ref{fig:Output}, while keeping the best result and terminating any new instantiations deviating from it or exceeding 2000 iterations. Choosing the 11 pixels having the largest standard deviation among the images in the training set, we were able to perturb the \emph{same} pixels to change the prediction of a digit '7' to that of any other digit (see Fig.~\ref{fig:onevsall}). While many instantiations do not converge to their destination, eventually successful runs manage to follow the full path (see Tab. \ref{tbl:confidence})

The main difference between the two algorithms is that Alg.~\ref{alg:basic} is deterministic, and thus any failure is fatal, whereas Alg.~\ref{alg:improved} is randomized and can use its extra degree of freedom to eventually avoid soft failures and find adversarial examples of Hamming distance $ m+1 $. Our experiments show that at least for the MNIST dataset, a Hamming distance of $ m+1 $ suffices to produce successful fully targeted attacks.

\begin{figure}[t]
	\vskip -0.0in
	\begin{center}
		\centerline{\includegraphics[width=\columnwidth]{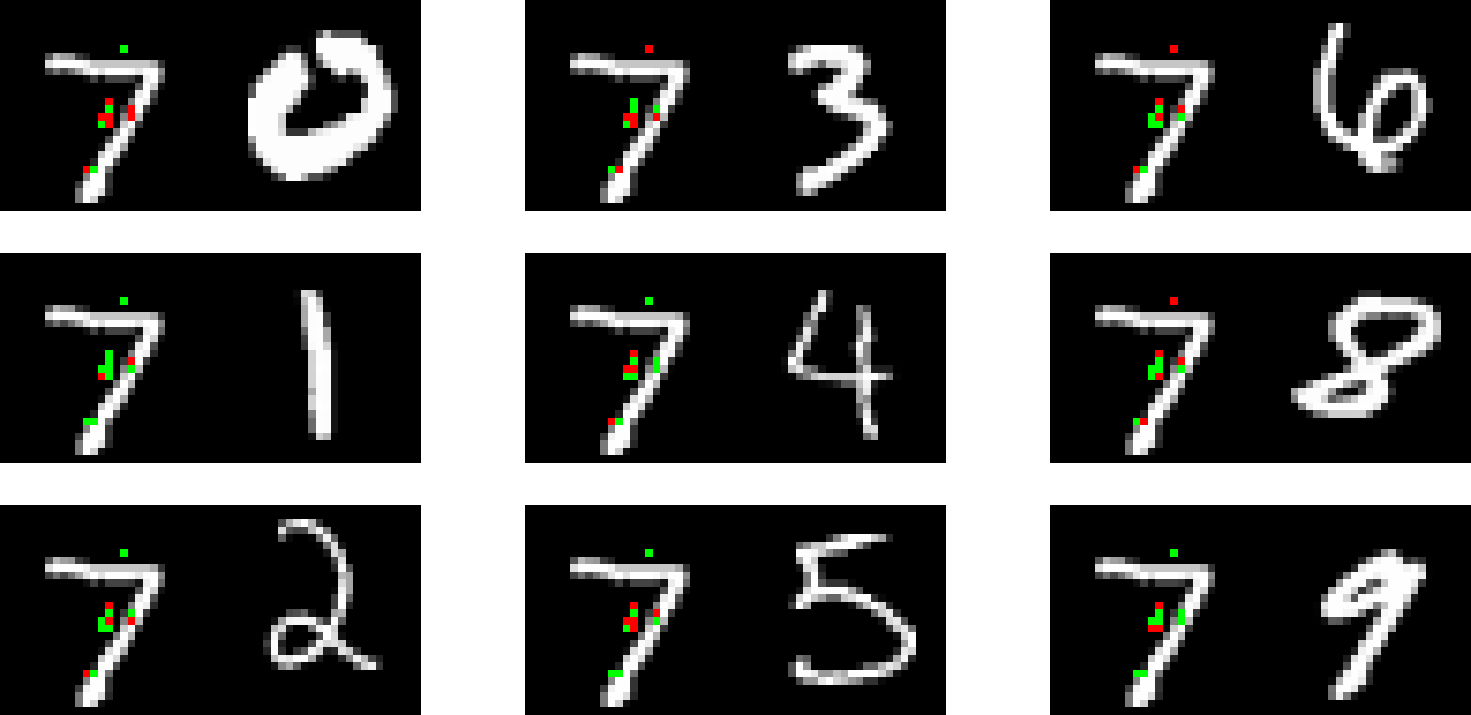}}
		\caption{Using the same set of 11 pixels, different perturbations change the input digit '7' to output the prediction vector of any other digit. Pixels in red have decreased values and pixels in green have increased values. Best viewed in color.}
		\label{fig:onevsall}
	\end{center}
	\vskip -0.2in
\end{figure}

\ifICML
\begin{center}
	\begin{table}\label{tbl:confidence}
		\begin{tabular}{l|c|c|c|c|c}
			\textbf{Digit}: & \textbf{0} & \textbf{1} & \textbf{2} & \textbf{3} & \textbf{4} \\ \hline
			Original    & -0.43      & 0.3        & 5.33       & 2.37       & -2.21      \\
			Target      & -10.62     & 1.96       & -0.99      & 18.69      & -1.06      \\
			Adversarial & -10.62     & 1.96       & -0.99      & 18.69      & -1.06
		\end{tabular}
		\vskip 0.2cm
		\begin{tabular}{l|c|c|c|c|c}
			\textbf{Digit}: & \textbf{5} & \textbf{6} & \textbf{7} & \textbf{8} & \textbf{9} \\ \hline
			Original    & 1.72       & -6.11      & 11.98      & -10.54     & -2.7       \\
			Target      & 6.78       & -20.27     & 4.52       & -8.33      & 8.84       \\
			Adversarial & 6.78       & -20.27     & 4.52       & -8.33      & 8.84
		\end{tabular}
		\caption{Prediction output values of the results when perturbing a '7' into a '3'.}
	\end{table}
\end{center}
\else
\begin{table}\label{tbl:confidence}
	\begin{center}
		\vspace{-2ex}
		\begin{tabular}{l|c|c|c|c|c|c|c|c|c|c}
			\textbf{Digit}: & \textbf{0} & \textbf{1} & \textbf{2} & \textbf{3} & \textbf{4}  & \textbf{5} & \textbf{6} & \textbf{7} & \textbf{8} & \textbf{9} \\ \hline
			Original    & -0.43      & 0.3        & 5.33       & 2.37       & -2.21      & 1.72       & -6.11      & 11.98      & -10.54     & -2.7       \\
			Target      & -10.62     & 1.96       & -0.99      & 18.69      & -1.06      & 6.78       & -20.27     & 4.52       & -8.33      & 8.84       \\
			Adversarial & -10.62     & 1.96       & -0.99      & 18.69      & -1.06 & 6.78       & -20.27     & 4.52       & -8.33      & 8.84
		\end{tabular}
		\vspace{2ex}
		\caption{Prediction output values of the results when perturbing a '7' into a '3'.}
	\end{center}
\end{table}
\fi

\section{Implications for Adversarial Examples  Mitigations}
\label{sec:Implications}

A natural conjecture that follows from our algorithms suggests that as long as
the transformation from the input space to the output space is piecewise
linear, one should expect (targeted) adversarial
examples of low $L_0$ norm.

Such a conjecture implies that any mitigation technique which does not change
the number of classes or the overall nature of piecewise linearity, is going
to be susceptible to (possibly new) adversarial examples. Hence, techniques
that add layers to the network
(e.g., autoencoders
as in~\cite{autoencoders}) or change the structure of the network, are
also going to be susceptible to new adversarial examples.

The conjecture is even more interesting in the context of adversarial
training~\cite{AdvTraining1,AdvTraining2}. While this approach seems to
offer an increased resilience to the same type of adversarial noise the
networks were trained to mitigate, new adversarial examples constructed
for the new networks exist. This behavior is easily explained by our
conjecture---the fact that the piecewise linear behavior and the number
of classes remains. Hence, one should expect new (targeted) adversarial examples
with low hamming distance even against these adversarially trained networks.

While the above ideas (and similar ones) seem to offer little resilience
to adversarial examples, we can point out two mitigation methods that do
not follow the conjecture, and thus may work.
The first is using image filters such
as the median filter as proposed in~\cite{DeepCAPTCHA}. These filters are
usually noncontinuous, resulting in breaking the input-space to output-space
translation into (many) isolated sub-problems.\footnote{We note that~\cite{DeepCAPTCHA} also proposes how to build adversarial examples when these filters are used. However, these adversarial examples are not of low Hamming distance.}
Another popular approach is to add random noise to the input~\cite{randomize},
thus altering the starting point $\pointx$ (as well as $\pointx'$). This
fuzziness in the exact location may affect the
accuracy of the network and  require
sufficiently large amount of entropy for the randomness to prevent
enumeration attempts~\cite{athalye2018obfuscated}.

\section{Conclusions and Open Problems}
\label{sec:Conclusions}

In this paper we developed a new way to think about the phenomenon of adversarial examples, and in particular we explained why we expect to find in our model adversarial examples with a Hamming distance of $m+1$ in neural networks which are designed to distinguish between $m$ possible classes. We experimentally verified this prediction using the MNIST dataset, where our algorithm failed to find any examples with a Hamming distance of $10$, but found a set of $11$ out of the $784$ pixels whose manipulation could change the prediction from one digit to any one of the other $9$ digits.

There are many interesting questions left open by our research:

\begin{enumerate}
	
	\item We use the fact that performing an arbitrarily large number of changes but in the same input variables does not increase the Hamming distance.
	Is there any alternative argument which can be used to analyze $L_2$ distances?
	
	\item In domain-specific examples such as images, we typically want to keep the modified values bounded. One way to achieve this is to use a larger Hamming distance, which allows us to change more input variables by smaller amounts. What are the possible tradeoffs between these two parameters?
	
	\item We explained why our algorithm never encounters a hard failure when the allowed Hamming distance is $m+1$ or higher. Is there a way to show that it does not encounter soft failures, in which it runs forever but converges only to an intermediate point along its desired trajectory?
	Can such points be avoided by a different choice of this trajectory?
	
	\item Our analysis shows that adversarial training of the neural network should have no effect on the existence of adversarial examples. What are all the other implications of our results?
	
\end{enumerate}

\subsection*{Acknowledgements}
We would like to thank Margarita Osadchy, Nathan Keller and Noga Alon for illuminating discussions, and Ben Feinstein for his valuable help with the experiment.

Orr Dunkelman is partially supported by the Center for Cyber 
Law \& Policy at the University of Haifa in conjunction with the Israel National Cyber Directorate in the Prime Minister’s Office. Eyal Ronen is partially supported by the Robert Bosch Foundation.

\Urlmuskip=0mu plus 1mu\relax

\bibliography{bib}
\bibliographystyle{icml2019}

\end{document}